\documentclass[oneside,reqno,english]{amsart}
\usepackage[T1]{fontenc}
\usepackage[latin9]{inputenc}
\usepackage{babel}
\usepackage{array}
\usepackage{refstyle}
\usepackage{float}
\usepackage{mathrsfs}
\usepackage{mathtools}
\usepackage{amstext}
\usepackage{amsthm}
\usepackage{amssymb}
\usepackage{stackrel}
\usepackage{graphicx}
\usepackage[all]{xy}
\usepackage[unicode=true,pdfusetitle,
 bookmarks=true,bookmarksnumbered=false,bookmarksopen=false,
 breaklinks=false,pdfborder={0 0 0},pdfborderstyle={},backref=false,colorlinks=false]
 {hyperref}
\hypersetup{
 colorlinks=true,citecolor=blue,linkcolor=blue,linktocpage=true}

\makeatletter

\AtBeginDocument{\providecommand\figref[1]{\ref{fig:#1}}}
\AtBeginDocument{\providecommand\thmref[1]{\ref{thm:#1}}}
\AtBeginDocument{\providecommand\secref[1]{\ref{sec:#1}}}
\AtBeginDocument{\providecommand\corref[1]{\ref{cor:#1}}}
\AtBeginDocument{\providecommand\lemref[1]{\ref{lem:#1}}}
\providecommand{\tabularnewline}{\\}
\RS@ifundefined{subsecref}
  {\newref{subsec}{name = \RSsectxt}}
  {}
\RS@ifundefined{thmref}
  {\def\RSthmtxt{theorem~}\newref{thm}{name = \RSthmtxt}}
  {}
\RS@ifundefined{lemref}
  {\def\RSlemtxt{lemma~}\newref{lem}{name = \RSlemtxt}}
  {}

\numberwithin{equation}{section}
\numberwithin{figure}{section}
\theoremstyle{plain}
\newtheorem{thm}{\protect\theoremname}[section]
\theoremstyle{plain}
\newtheorem{lem}[thm]{\protect\lemmaname}
\theoremstyle{remark}
\newtheorem{rem}[thm]{\protect\remarkname}
\theoremstyle{definition}
\newtheorem{example}[thm]{\protect\examplename}
\theoremstyle{definition}
\newtheorem*{problem*}{\protect\problemname}
\theoremstyle{plain}
\newtheorem{cor}[thm]{\protect\corollaryname}
\theoremstyle{remark}
\newtheorem*{note*}{\protect\notename}
\theoremstyle{definition}
\newtheorem{defn}[thm]{\protect\definitionname}

\providecommand{\MR}[1]{}

\usepackage{bbm}
\allowdisplaybreaks
\usepackage{needspace}
\usepackage{refstyle}

\newref{lem}{refcmd={Lemma \ref{#1}}}
\newref{thm}{refcmd={Theorem \ref{#1}}}
\newref{cor}{refcmd={Corollary \ref{#1}}}
\newref{sec}{refcmd={Section \ref{#1}}}
\newref{subsec}{refcmd={Section \ref{#1}}}
\newref{chap}{refcmd={Chapter \ref{#1}}}
\newref{prop}{refcmd={Proposition \ref{#1}}}
\newref{exa}{refcmd={Example \ref{#1}}}
\newref{tab}{refcmd={Table \ref{#1}}}
\newref{rem}{refcmd={Remark \ref{#1}}}
\newref{def}{refcmd={Definition \ref{#1}}}
\newref{fig}{refcmd={Figure \ref{#1}}}
\newref{claim}{refcmd={Claim \ref{#1}}}

\makeatletter
\newcommand{\xyR}[1]{%
\makeatletter
\xydef@\xymatrixrowsep@{#1}
\makeatother
}

\makeatletter
\newcommand{\xyC}[1]{%
\makeatletter
\xydef@\xymatrixcolsep@{#1}
\makeatother
}

\@ifundefined{showcaptionsetup}{}{%
 \PassOptionsToPackage{caption=false}{subfig}}
\usepackage{subfig}
\makeatother

\providecommand{\corollaryname}{Corollary}
\providecommand{\definitionname}{Definition}
\providecommand{\examplename}{Example}
\providecommand{\lemmaname}{Lemma}
\providecommand{\notename}{Note}
\providecommand{\problemname}{Problem}
\providecommand{\remarkname}{Remark}
\providecommand{\theoremname}{Theorem}

\begin{document}
\title{Operator theory, kernels, and feedforward neural networks}
\dedicatory{Dedicated to the memory of Professor Ka-Sing Lau.}
\author{Palle E.T. Jorgensen}
\address{(Palle E.T. Jorgensen) Department of Mathematics, The University of
Iowa, Iowa City, IA 52242-1419, U.S.A.}
\email{palle-jorgensen@uiowa.edu}
\author{Myung-Sin Song}
\address{(Myung-Sin Song) Department of Mathematics and Statistics, Southern
Illinois University Edwardsville, Edwardsville, IL 62026, USA}
\email{msong@siue.edu}
\author{James Tian}
\address{(James F. Tian) Mathematical Reviews, 416 4th Street Ann Arbor, MI
48103-4816, U.S.A.}
\email{jft@ams.org}
\begin{abstract}
In this paper we show how specific families of positive definite kernels
serve as powerful tools in analyses of iteration algorithms for multiple
layer feedforward Neural Network models. Our focus is on particular
kernels that adapt well to learning algorithms for data-sets/features
which display intrinsic self-similarities at feedforward iterations
of scaling.
\end{abstract}


\maketitle
\tableofcontents{}

\section{Introduction}

Recently many authors have offered diverse approaches to feedforward
Neural Network (NN) algorithms \cite{MR4505410,MR4505203,MR4496811,MR4491367},
as well as optimization terms based on kernels. Here we establish
some new results in operator theory, and we bring them to bear on
the problem. The list of applications of feedforward NN models includes
a variety of machine learning settings, and deep NN based on kernels
\cite{MR4505882,MR4503771,MR4492099,MR4500409,MR4376564,MR4268857,MR4134776,MR4131039}.

A common theme in feedforward NN models is specific prescribed iterations
which entail (i) ReLu functions \cite{MR4473797,MR4409717,MR4399726,MR4476907,MR4468133,MR4458444},
(ii) substitution from prescribes systems of affine mappings. Moreover,
(iii) each step is then linked to the next with a choice of an activation
function. In this paper we show that there are natural positive definite
kernels associated with the three steps going into feedforward NN
constructions, as well as to their iteration. We believe that this
then yields a more direct tool for kernel-based feedforward NN models.
This advantage of our approach is based on two facts. First, we identify
a direct notion of kernel iteration which accounts for traditional
function theoretic feedforward NN steps. Secondly, our approach offers
a more direct and natural choices of kernels which govern approximations
involved in deep NN models, for example graph NN constructions.

While positive definite kernels and their associated reproducing kernel
Hilbert spaces have found diverse applications in pure and applied
mathematics, we shall focus here on a new role of kernels in feedforward
network models. In more detail, the main purpose of our paper is a
presentation of choices of particular families of positive definite
kernels which serve as powerful tools in analyses of multiple layer
feedforward Neural Networks.

In general, reproducing kernel constructions, and the corresponding
RKHSs are powerful tools in diverse applications. In the present framework
of kernel neural networks (KNNs) , their role may be summarized as
follows: Starting with the problem at hand, when we build our RKHS$(\mu)$
via IFS iterations (e.g., via Cantor-like fractal limits), then the
Cantor-like $\mu$ activation functions arise as relative reproducing
dipole functions for RKHS$(\mu)$ as in \figref{rmu} below.

\section{\label{sec:nn}Neural networks (NN), and reproducing kernel Hilbert
spaces (RKHS)}

A main theme in our paper is a development of new tools for design
of feedforward Neural Network constructs. For this purpose we point
out the use of positive definite kernels, and associated generating
function for the NN algorithms. These kernel based functions include
the more familiar ReLu functions, see Theorems \ref{thm:hk} and \ref{thm:hmu}
below. We stress that the particular RKHS constructs will be relative
in the sense of Theorems \ref{thm:hk} and \ref{thm:hmu}, i.e., the
inner product reproduces differences of function values. 

Our approach to the use of kernels and functions for feedforward Neural
Network (NN) algorithms, is based on a systematic study of two classes
of operators. They act as follows: (i) between prescribed kernel Hilbert
spaces, and (ii) other operators acting at indexed levels in the network,
i.e., operators acting at fixed levels, so within choices of kernel
Hilbert spaces. Case (i) includes a systematic study of composition
operators (see Corollaries \ref{cor:b7} and \ref{lem:b9}) in the
context of kernel Hilbert spaces; and case (ii), the study of multiplier
operators and their adjoints, see e.g., \thmref{b15}. We emphasize
that the two classes of operations discussed below depend on choices
of kernels at each level in particular NN-network models. Together
these families of operators allow for realizations of black box filter-entries
in associated generalized multi-resolutions systems, including operators
which consist of composition followed by multiplication. Specific
3D applications are presented in the subsequent sections, secs \ref{sec:kact}
and \ref{sec:fwidth}.

\textbf{Conventions. }Inside the paper we shall work with Hilbert
spaces of functions, e.g., reproducing kernel Hilbert spaces (RKHSs),
$L^{2}$ spaces, and Sobolev spaces. It will be assumed that these
are Hilbert spaces of real valued functions. Inner products will be
written $\left\langle \cdot,\cdot\right\rangle $, and we shall use
subscripts on $\left\langle \cdot,\cdot\right\rangle $ to indicate
the Hilbert space under consideration. Moreover, in our use of differentiation,
or differential operators, we shall mean weak derivatives, i.e., differentiation
in the sense of distributions, or making use of the natural duality
for the spaces under consideration. Our restriction here to the real
valued case is dictated by our present applications to feedforward
Neural Networks. However, many of our general results in \secref{nn}
below extend to complex RKHS theory. The latter in turn are important
in the study of geometry and potential theory of complex domains,
see e.g., \cite{MR1340173}. 

The power of kernel machines derives in part from the following facts.
First, kernel machines serve to map points in a low-dimensional data
sets (typically nonlinear) into higher dimensions. The dimensionality
of this linear \textquotedblleft hyperspace\textquotedblright{} may
be infinite but is designed for optimization and efficient encoding
of features. Hence the kernel method allows one to find coefficients
of separating hyperplanes for the problem at hand via RKHS-inner products,
one selected for each pair of high-dimensional features. While kernel
machines of various types have been used for decades, it was with
the invention of support vector machines (SVMs) that kernels have
now taken center stage (see e.g., \cite{zbMATH01669138,MR4329806,MR2849119,MR3108145,MR2274418,MR2246374}).
By now, SVMs are used in diverse applications, including in bioinformatics
(for finding similarities between different protein sequences), machine
vision, and handwriting recognition. Deep neural networks (to be discussed
in Sections \ref{sec:kact} and \ref{sec:fwidth} below) are made
of layers of artificial neurons: input layer, an output layer, and
multiple hidden layers in-between them. Deeper the networks have more
hidden layers. The parameters of the network represent the strengths
of the connections between layers. Training a network yield determination
of values of parameters. Once trained, the ANN represents a model
for turning an input (say, an image) into an output (a label or category).

The variety of uses of forward Neural Network algorithms, the recent
literature is substantial and diverse, especially with regards to
applications. See e.g., \cite{MR4512468,MR4505888,MR4395164,MR4185345,MR4072078,MR3457582}.

The following lemma is a basic result in the theory of RKHSs. For
details, see e.g., \cite{MR738131,zbMATH06526193,zbMATH06526192,MR4250453,Szabook},
and also \cite{MR4295177} and the papers cited therein. 
\begin{lem}
\label{lem:B1}Fix a p.d. kernel $X\times X\xrightarrow{\;K\;}\mathbb{R}\left(\text{or \ensuremath{\mathbb{C}}}\right)$,
let $\mathscr{H}_{K}$ denote the corresponding RKHS. Then a function
$F$ on $X$ is in $\mathscr{H}_{K}$ if and only if there exists
a constant $C_{F}<\infty$, such that the following estimate holds
for all $n\in\mathbb{N}$, all $\left(\xi_{i}\right)_{i=1}^{n}$,
$\xi_{i}\in\mathbb{R}\left(\text{or \ensuremath{\mathbb{C}}}\right)$,
and all $\left(x_{i}\right)_{i=1}^{n}$, $x_{i}\in X$: 
\begin{equation}
\left|\sum_{i=1}^{n}\xi_{i}F\left(x_{i}\right)\right|^{2}\leq C_{F}\sum_{i=1}^{n}\sum_{j=1}^{n}\overline{\xi}_{i}\xi_{j}K\left(x_{i},x_{j}\right).\label{eq:rhks1}
\end{equation}
\end{lem}

\begin{rem}
With the construction $K\mapsto\mathscr{H}_{K}$ (referring to a RKHS
of a fixed p.d. kernel $K$), we arrive at the following two conclusions:
\begin{enumerate}
\item For all $x\in X$, the function $K_{x}:=K\left(\cdot,x\right)$ is
in $\mathscr{H}_{K}$; and
\item For all $F\in\mathscr{H}_{K}$, and $x\in X$, we have 
\begin{equation}
F\left(x\right)=\left\langle F,K\left(\cdot,x\right)\right\rangle _{\mathscr{H}_{K}},\label{eq:b2}
\end{equation}
i.e., the values of functions $F$ in $\mathscr{H}_{K}$ are \emph{reproduced}
via the inner product $\left\langle \cdot,\cdot\right\rangle _{\mathscr{H}_{K}}$,
and the kernel functions. 
\end{enumerate}
In addition to (\ref{eq:b2}), we shall also consider \emph{relative
reproducing kernels}, and \emph{relative} RKHSs. As noted in \cite{MR3251728},
the \emph{relative }reproducing property takes the following form
\begin{equation}
F\left(y\right)-F\left(x\right)=\left\langle F,v_{x,y}\left(\cdot\right)\right\rangle _{\mathscr{H}_{rel}},\label{eq:b3}
\end{equation}
now valid for all pairs of points $x,y\in X$. So this entails double-indexed
kernel functions $v_{x,y}\in\mathscr{H}_{rel}$. 

A particular class of $\mathscr{H}_{rel}$ spaces are considered in
\thmref{hmu} below. There the setting is $X=\mathbb{R}$, and the
relative kernel functions $v_{a,b}$ take the form of activation functions
for classes of feedforward-NN-algorithms, see e.g., \figref{rmu}.

A systematic study of (\ref{eq:b3}) is undertaken in \cite{MR3251728}
where it is shown that the setting of relative reproducing is characterized
by \emph{conditionally negative definite functions}. 
\end{rem}

We now recall the RKHS for the standard 1-dimensional Brownian motion.
(See e.g., \cite{MR4295177,MR4302453,MR4274591,MR4472250}.)
\begin{lem}
When $K$ is the Brownian motion kernel on $\mathbb{R}_{\geq}\times\mathbb{R}_{\geq}$,
i.e., 
\begin{equation}
K\left(x,y\right)=x\wedge y=\frac{\left|x\right|+\left|y\right|-\left|x-y\right|}{2},\quad x,y\in\mathbb{R}_{\geq},\label{eq:B2}
\end{equation}
the corresponding RKHS $\mathscr{H}_{K}$ is the Hilbert space of
absolutely continuous functions $f$ on $\mathbb{R}$ such that the
derivative $f'=df/dx$ is in $L^{2}\left(\mathbb{R}\right)$, and
$f\left(0\right)=0$. Moreover, 
\begin{equation}
\left\Vert f\right\Vert _{\mathscr{H}_{K}}^{2}=\int_{\mathbb{R}_{\geq}}\left|f'\left(x\right)\right|^{2}dx,\quad\text{for all \ensuremath{f\in\mathscr{H}_{K}}.}\label{eq:1}
\end{equation}
\end{lem}

\begin{proof}
The key observation is that, if $x>0$, the function 
\begin{equation}
\mathbb{R}_{\geq}\ni y\longmapsto F_{x}\left(y\right):=K\left(y,x\right)=\begin{cases}
y & \text{if \ensuremath{y\leq x}}\\
x & \text{if \ensuremath{y>x}}
\end{cases}
\end{equation}
has weak derivative. Indeed, we have 
\begin{equation}
\frac{dF_{x}}{dy}=\chi_{\left[0,x\right]},
\end{equation}
i.e., the indicator function of the interval $\left[0,x\right]$.
Hence if $f$ is a function with $f'\in L^{2}\left(\mathbb{R}\right)$
and $f\left(0\right)=0$, then 
\begin{equation}
f\left(x\right)=f\left(x\right)-f\left(0\right)=\int_{0}^{x}f'\left(y\right)dy=\int_{\mathbb{R}}F_{x}'\left(y\right)f'\left(y\right)dy,\label{eq:a4}
\end{equation}
and the RHS in (\ref{eq:a4}) is the inner product from the Hilbert
space defined by the RHS in (\ref{eq:1}). 

The corresponding implication follows from the general theory of RKHS.
Recall that the RKHS of a kernel is a Hilbert space completion of
the functions 
\begin{equation}
y\longmapsto K\left(x,y\right)
\end{equation}
as $x$ varies over $\mathbb{R}$. Moreover, for $K\left(x,y\right)=x\wedge y$,
\[
\left\langle K\left(\cdot,x_{1}\right),K\left(\cdot,x_{2}\right)\right\rangle _{\mathscr{H}_{K}}=K\left(x_{1},x_{2}\right)=x_{1}\wedge x_{2},
\]
and we can compute as follows: 
\begin{align*}
\int_{\mathbb{R}}\left(\frac{d}{dy}K\left(\cdot,x_{1}\right)\right)\left(\frac{d}{dy}K\left(\cdot,x_{2}\right)\right)dy & =\int_{\mathbb{R}}\chi_{\left[0,x_{1}\right]}\left(y\right)\chi_{\left[0,x_{2}\right]}\left(y\right)dy\\
 & =\lambda\left(\left[0,x_{1}\right]\cap\left[0,x_{2}\right]\right)\\
 & =x_{1}\wedge x_{2}=K\left(x_{1},x_{2}\right)
\end{align*}
where $\lambda=dy$ denotes the Lebesgue measure. 
\end{proof}
\begin{rem}
Note that the functions $v_{a,b}$ (called dipoles) in $\mathscr{H}_{K}$
which satisfy 
\[
f\left(b\right)-f\left(a\right)=\left\langle f,v_{a,b}\right\rangle ,\quad\text{for all \ensuremath{f\in\mathscr{H}_{K}}}
\]
(see (\ref{eq:1}) and (\ref{eq:a4})) are as follows:
\[
v_{a,b}\left(x\right)=\begin{cases}
0 & \text{if }x<a\\
x-a & \text{if }a\leq x<b\\
b-a & \text{if }x>b,
\end{cases}
\]
as illustrated in \figref{dp}. Also compare with \thmref{hmu} and
\figref{rmu}, and the iterations in \secref{fwidth}.
\end{rem}

\begin{figure}[H]
\includegraphics[width=0.5\textwidth]{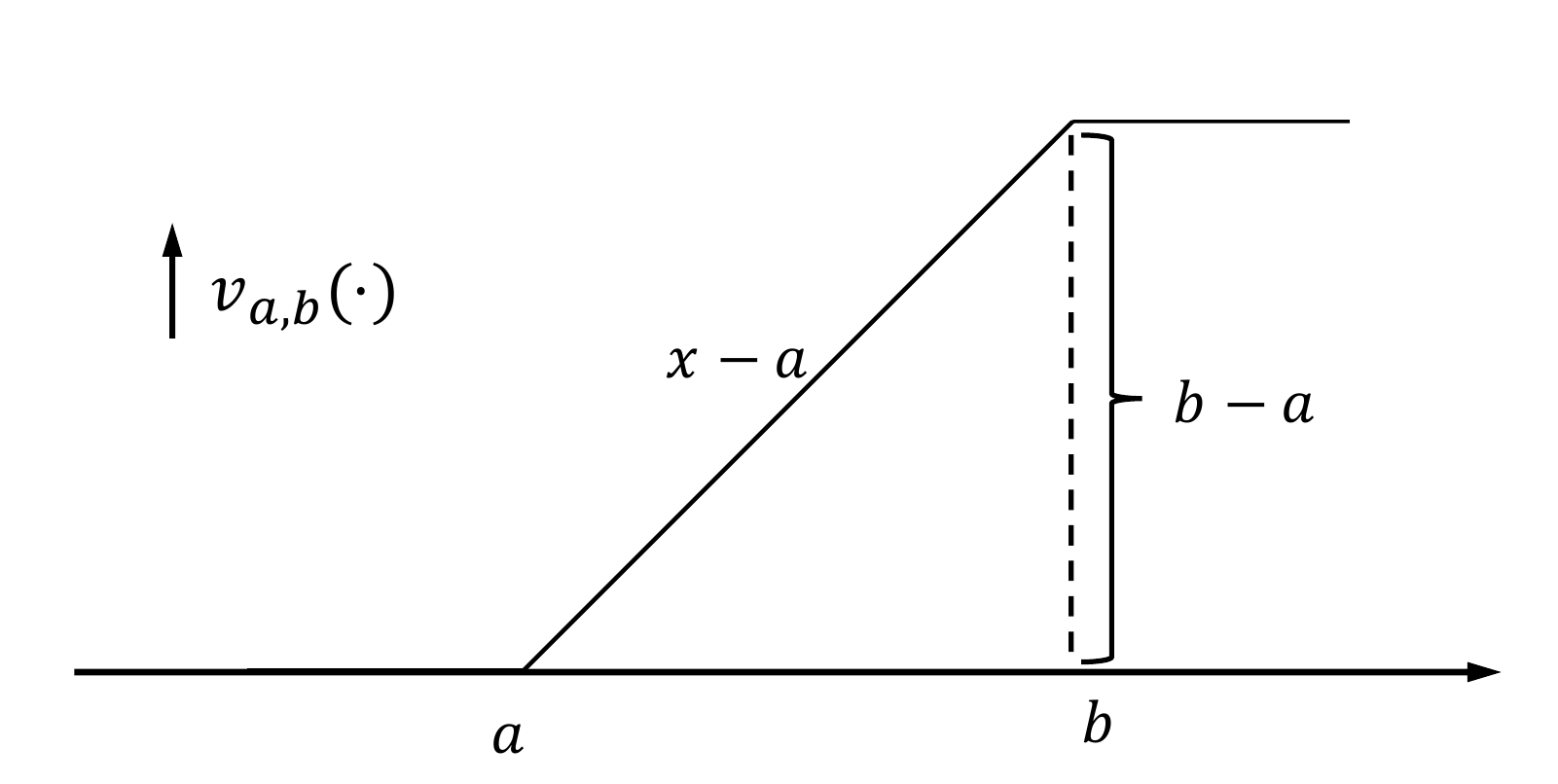}

\caption{\label{fig:dp}The generating dipole function $\left\{ v_{a,b}\right\} $
indexed by pairs $a,b$ such that $a<b$. Compare with \figref{rmu}
below.}
\end{figure}

\textbf{Induced metrics}

For a general p.d. kernel $K$ on $X\times X$, there is an induced
metric on $X$, 
\[
d_{K}:X\times X\rightarrow\mathbb{R}_{+}
\]
defined as (see e.g., \cite{MR4302453})
\begin{equation}
d_{K}\left(x,y\right)=\left\Vert K\left(\cdot,x\right)-K\left(\cdot,y\right)\right\Vert _{\mathscr{H}_{K}}^{2}.
\end{equation}
In particular, 
\[
d_{K}\left(x,y\right)=K\left(x,x\right)+K\left(y,y\right)-2\Re\left\{ K\left(x,y\right)\right\} .
\]
Note that $d_{K}^{1/2}$ is also a metric on $X\times X$. 
\begin{example}
For $K\left(x,y\right)=x\wedge y$ on $\mathbb{R}\times\mathbb{R}$
as in (\ref{eq:B2}), 
\begin{align*}
\left\Vert K\left(\cdot,s\right)-K\left(\cdot,t\right)\right\Vert _{\mathscr{H}_{K}}^{2} & =\left\Vert \left(\cdot\wedge s\right)'-\left(\cdot\wedge t\right)'\right\Vert _{L^{2}}^{2}\\
 & =\left\Vert \chi_{\left[0,s\right]}-\chi_{\left[0,t\right]}\right\Vert _{L^{2}}^{2}\\
 & =\left|s-t\right|.
\end{align*}
\end{example}

The results below deal with a general framework of pairs of sets $X$
and $Y$, each equipped with a positive definite kernel, $K$ resp.,
$L$, $K$ on $X$, and $L$ on $Y$. With view to realization of
feedforward Neural Network-functions, we will present an explicit
framework (see (\ref{eq:b9}) and (\ref{eq:ad1})) which allows us
to pass from (nonlinear) functions $f:X\rightarrow Y$ to linear operators
$T_{f}$ acting between the respective RKHSs $\mathscr{H}_{K}$ and
$\mathscr{H}_{L}$. This will be a representation in the sense that
composition of functions will map into products of the corresponding
linear operators. Some care must be exercised as the linear operators
$T_{f}$ will in general be unbounded. Nonetheless, we shall show
that the operators still fall in a class where spectral resolutions
are available, see \thmref{b9} and \corref{b10}. 
\begin{thm}
\label{thm:mc}Consider p.d. kernels $X\times X\xrightarrow{\;K\;}\mathbb{R}$
and $Y\times Y\xrightarrow{\;L\;}\mathbb{R}$. Let $f:X\rightarrow Y$
be Lipschitz continuous with respect to the induced metrics $d_{K},d_{L}$,
i.e., 
\[
d_{L}\left(f\left(x\right),f\left(y\right)\right)\leq c_{f}d_{K}\left(x,y\right),\quad x,y\in X,
\]
for some constant $c_{f}$. Define the operator $T_{f}:\mathscr{H}_{K}\rightarrow\mathscr{H}_{L}$
by 
\[
T_{f}\left(K_{x}\right)\left(y\right)=L\left(f\left(x\right),y\right)
\]
and extend it by linearity and density. 

Then, for any fixed $c\in Y$, the function 
\begin{equation}
F_{f}:X\rightarrow\mathbb{R},\quad F_{f}\left(x\right):=L\left(f\left(x\right),c\right)\label{eq:b9}
\end{equation}
is in the RKHS $\mathscr{H}_{K}$ if, and only if 
\begin{equation}
L_{c}\in dom(T_{f}^{*}),
\end{equation}
the domain of the adjoint operator. 

Moreover, 
\begin{align*}
F_{f,y}\left(x\right) & :=L\left(f\left(x\right),y\right)\in\mathscr{H}_{K},\;\forall y\in Y\\
 & \Updownarrow\\
T_{f}\: & \text{is closable}.
\end{align*}
(See also \thmref{mc2}.)
\end{thm}

\begin{proof}
Let the setting be as in the statement of the theorem, i.e., $X\xrightarrow{\;f\;}Y$
assumed continuous with respect to the two metrics, $d_{K}$ on $X$
and $d_{L}$ on $Y$; so in particular, for pairs of points $x_{1},x_{2}\in X$,
we have 
\begin{align}
d_{K}\left(x_{1},x_{2}\right) & =\left\Vert K\left(\cdot,x_{1}\right)-K\left(\cdot,x_{2}\right)\right\Vert _{\mathscr{H}_{K}}^{2}\label{eq:d1}\\
 & =K\left(x_{1},x_{1}\right)+K\left(x_{2},x_{2}\right)-2K\left(x_{1},x_{2}\right).
\end{align}
We further fix a point $c\in Y$, and set $F=F_{f,c}$, specified
as follows:
\[
F\left(x\right)=L\left(f\left(x\right),c\right),\;\text{for all \ensuremath{x\in X},}
\]
so $F:X\rightarrow\mathbb{R}_{+}\cup\left\{ 0\right\} $. 

Now, for every $N$, and every subset $S_{N}=\left(x_{1},x_{2},\dots,x_{N}\right)\subset X$,
consider the following matrix operations (in $N$ dimensions): 
\begin{equation}
\underset{\text{column vectors}}{\underbrace{F\big|_{N}:=\begin{bmatrix}F\left(x_{1}\right)\\
\vdots\\
F\left(x_{N}\right)
\end{bmatrix}}},\;\text{and}\quad\underset{\text{matrix of a rank-1 operator}}{\underbrace{Q_{N}:=\left|F\big|_{N}\left\rangle \right\langle F\big|_{N}\right|}},
\end{equation}
i.e., the rank-1 operator on $\mathbb{R}^{N}$ written in Dirac's
notation, defined as 
\begin{equation}
Q_{N}\left(\xi\right)=\left\langle F_{N},\xi\right\rangle F_{N}
\end{equation}
for all $\xi\in\mathbb{R}^{N}$. Set 
\begin{equation}
K_{N}:=\left(K\left(x_{i},x_{j}\right)\right)_{i,j=1}^{N}=\begin{bmatrix}K\left(x_{1},x_{1}\right) & \cdots & K\left(x_{1},x_{N}\right)\\
\vdots & \ddots & \vdots\\
K\left(x_{N},x_{1}\right) & \cdots & K\left(x_{N},x_{N}\right)
\end{bmatrix},
\end{equation}
a sample matrix. 

For the convex cone of all positive definite $N\times N$ matrices,
we introduce the following familiar ordering, $K\ll_{C}K'$ iff (Def.)
$\exists C<\infty$ such that 
\begin{equation}
\xi^{T}K_{N}\xi\leq C\xi^{T}K'_{N}\xi\;\text{for all }\xi\in\mathbb{R}^{N}.
\end{equation}

Now an application of \lemref{B1} above shows that the assertion
in the theorem is equivalent to the existence of a finite constant
$C$ (independent of $S_{N}=\left(x_{i}\right)_{i=1}^{N}$) satisfying
$Q_{N}\ll_{C}K_{N}$, i.e., the estimate 
\begin{equation}
\left|\sum_{i}\xi L\left(f\left(x_{i}\right),c\right)\right|^{2}\leq C\sum_{i}\sum_{j}\xi_{i}K\left(x_{i},x_{j}\right)\xi_{j}
\end{equation}
holds for all $N$, all $S_{N}=\left\{ x_{i}\right\} _{i=1}^{N}$,
and all $\xi\in\mathbb{R}^{N}$. We get this from the assumption on
$f$ in the theorem. See details below. 
\end{proof}
\textbf{Summary of \thmref{mc}:} Start with $K$ p.d. on $X\times X$,
$L$ p.d. on $Y\times Y$, and $f:X\rightarrow Y$. We introduce the
metrics $d_{K}$ on $X$, and $d_{L}$ on $Y$, and we consider $f$
continuous, or Lipchitz. To get the desired conclusion 
\[
\left(x\longmapsto L\left(f\left(x\right),y\right)\right)\in\mathscr{H}_{K},
\]
we must introduce an operator $T_{f}:\mathscr{H}_{K}\rightarrow\mathscr{H}_{L}$.
The right choice is 
\[
L_{y}\in dom(T_{f}^{*}).
\]
See details below: 

Fixing two kernels $K$ and $L$, assumed p.d. on $X\times X$, and
on $Y\times Y$. Pass to the corresponding RKHS $\mathscr{H}_{K}$
and $\mathscr{H}_{L}$. 
\begin{problem*}
Find conditions on functions $X\xrightarrow{\;f\;}Y$ with the property
that, for $\forall y\in Y$, then the induced function 
\begin{equation}
\underset{F_{f,y}\left(\cdot\right)\text{ as a function on \ensuremath{X}}}{\underbrace{\left(X\ni x\longmapsto L\left(f\left(x\right),y\right)\right)}}\in\mathscr{H}_{K}.\label{eq:ad1}
\end{equation}
\end{problem*}
The argument stressed below is via dual operators (bounded) 
\[
\xymatrix{\mathscr{H}_{K}\ar@/{}^{1.3pc}/[rr]^{T_{f}} &  & \mathscr{H}_{L}\ar@/{}^{1.3pc}/[ll]^{T_{f}^{*}}}
;
\]
but the unbounded case is also interesting. 

Some remarks on the definition of the operator $T_{f}:\mathscr{H}_{K}\rightarrow\mathscr{H}_{L}$
in the case when no additional assumptions are placed on $X\xrightarrow{\;f\;}Y$. 

We define 
\[
T_{f}\left(K_{x}\right)\left(y\right)=L\left(f\left(x\right),y\right);
\]
and so we extend $T_{f}$ to linear combinations: 
\begin{equation}
\mathscr{D}_{K}:=\underset{\text{function on \ensuremath{X}}}{\big\{\underbrace{\sum_{i}c_{i}K_{x_{i}}}\big\}}\xrightarrow{\quad T_{f}\quad}\underset{\text{function on \ensuremath{Y}}}{\underbrace{\sum_{i}c_{i}L\left(f\left(x_{i}\right),\cdot\right)}}.\label{eq:ad2}
\end{equation}
But to make sense of (\ref{eq:ad2}) so it is well defined, we must
be careful with equivalence classes. 

If $f:X\rightarrow Y$ is a general function, the (generalized) operator
(\ref{eq:ad2}) may be non-closable. However, we can still define
the adjoint $T_{f}^{*}$ , but its domain might be ``small''.

Set 
\begin{equation}
\mathscr{D}_{K}:=span\left\{ K_{x}\right\} _{x\in X},\label{eq:ad3}
\end{equation}
then (Definition) a vector $\psi\in\mathscr{H}_{L}$ is in $dom(T_{f}^{*})$
iff $\exists C_{\psi}<\infty$ s.t. 
\begin{equation}
\left|\left\langle T_{f}\varphi,\psi\right\rangle _{\mathscr{H}_{L}}\right|\leq C_{\psi}\left\Vert \varphi\right\Vert _{\mathscr{H}_{K}},\quad\forall\varphi\in\mathscr{D}_{K}.\label{eq:ad4}
\end{equation}
We then set $T_{f}^{*}\psi=$ the solution to 
\begin{equation}
\left\langle T_{f}\varphi,\psi\right\rangle _{\mathscr{H}_{L}}=\left\langle \varphi,T_{f}^{*}\psi\right\rangle _{\mathscr{H}_{K}},\quad\forall\varphi\in\mathscr{D}_{K}.\label{eq:ad5}
\end{equation}

Let $K,L,f$ be as specified, and assume for some $y\in Y$, that
we have $L_{y}\in dom(T_{f}^{*})$, then (\ref{eq:ad5}) for $L_{y}\in\mathscr{H}_{L}\xrightarrow{\;T_{f}^{*}\;}\mathscr{H}_{K}$,
$\varphi=K_{x}$, $\psi=L_{y}$ yields 
\[
x\longmapsto T_{f}^{*}\left(L_{y}\right)\left(x\right)=L\left(f\left(x\right),y\right)\in\mathscr{H}_{K}.
\]
So the conclusion in \thmref{mc} that 
\begin{equation}
\left(x\longmapsto L\left(f\left(x\right),y\right)\right)\in\mathscr{H}_{K}\label{eq:ad6}
\end{equation}
holds iff $L_{y}\in dom(T_{f}^{*})$. 

In this case there are no difficulties with (\ref{eq:ad2}) and we
get a dual pair $T_{f}$ and $T_{f}^{*}$, 
\begin{equation}
\left\langle T_{f}\varphi,\psi\right\rangle _{\mathscr{H}_{L}}=\left\langle \varphi,T_{f}^{*}\psi\right\rangle _{\mathscr{H}_{K}}\label{eq:ad7}
\end{equation}
for $\forall\varphi\in\mathscr{D}_{K}$, and $\psi\in\mathscr{D}_{L}$,
$\xymatrix{\mathscr{H}_{K}\ar@/{}^{0.7pc}/[r]^{T_{f}} & \mathscr{H}_{L}\ar@/{}^{0.7pc}/[l]^{T_{f}^{*}}}
$. 

Setting $\varphi=K_{x}$, and $\psi=L_{y}$, (\ref{eq:ad7}) implies
\begin{equation}
T_{f}^{*}\left(L_{y}\right)\left(x\right)=L\left(f\left(x\right),y\right)=\left(T_{f}\left(K_{x}\right)\right)\left(y\right).\label{eq:ad8}
\end{equation}
But the previous condition $L_{y}\in dom(T_{f}^{*})$ (compare (\ref{eq:ad7}))
amounts to the assertion that $T_{f}^{*}\left(L_{y}\right)\in\mathscr{H}_{K}$,
and by (\ref{eq:ad8}), this is then the conclusion for \thmref{mc}. 
\begin{cor}[composition operators]
\label{cor:b7} Let $X$, $Y$, $K$ and $L$ be as specified above;
in particular, $K$ is a fixed p.d. kernel on $X\times X$, and the
RKHS $\mathscr{H}_{K}$ is a Hilbert space of scalar valued functions
on $X$. Similarly, $\mathscr{H}_{L}$ is a Hilbert space of scalar
valued functions on $Y$. Both $\mathscr{H}_{K}$ and $\mathscr{H}_{L}$
satisfy the defining axioms for RKHSs; see \lemref{B1} above. As
noted, every function $f$, $X\xrightarrow{\;f\;}Y$, induces a linear
operator 
\begin{equation}
\mathscr{H}_{K}\xrightarrow{\;T_{f}\;}\mathscr{H}_{L},\label{eq:st1}
\end{equation}
with dense domain $\mathscr{D}_{K}$; see the statement of \thmref{mc}.
For the adjoint operator $T_{f}^{*}$, 
\begin{equation}
\mathscr{H}_{L}\xrightarrow{\;T_{f}^{*}\;}\mathscr{H}_{K},\label{eq:st2}
\end{equation}
we have the following: For a function $\psi$ in $\mathscr{H}_{L}$,
the two characterizations (\ref{eq:st3a}) and (\ref{eq:st3b}) hold:
\begin{gather}
\psi\in dom(T_{f}^{*})\label{eq:st3a}\\
\Updownarrow\nonumber \\
\psi\circ f\in\mathscr{H}_{K}.\label{eq:st3b}
\end{gather}
In the affirmative, 
\begin{equation}
T_{f}^{*}\left(\psi\right)=\psi\circ f:X\rightarrow\mathbb{R},\label{eq:st4}
\end{equation}
i.e., $T_{f}^{*}$ is the composition operator. 
\end{cor}

\begin{proof}
(\ref{eq:st3a})$\Rightarrow$(\ref{eq:st3b}). Assume (\ref{eq:st3a}),
we then apply (\ref{eq:ad4}), and get $C_{\psi}<\infty$ with: 
\begin{equation}
\left|\sum_{i}c_{i}\left(T_{f}^{*}\left(\psi\right)\right)\left(x_{i}\right)\right|^{2}\leq C_{\psi}\sum_{i}\sum_{j}c_{i}c_{j}K\left(x_{i},x_{j}\right).\label{eq:st5}
\end{equation}
But 
\begin{align}
T_{f}^{*}\left(\psi\right)\left(x_{i}\right) & =\left\langle K_{x_{i}},T_{f}^{*}\left(\psi\right)\right\rangle _{\mathscr{H}_{K}}\label{eq:st6}\\
 & =\left\langle T_{f}\left(K_{x_{i}}\right),\psi\right\rangle _{\mathscr{H}_{L}}\nonumber \\
 & =\left\langle L\left(f\left(x_{i}\right),\cdot\right),\psi\right\rangle _{\mathscr{H}_{L}}\nonumber \\
 & =\psi\left(f\left(x_{i}\right)\right),\nonumber 
\end{align}
where we used the RKHS property for $\mathscr{H}_{L}$ in the last
step. Substitution into (\ref{eq:st5}) yields
\begin{equation}
\left|\sum_{i}c_{i}\psi\left(f\left(x_{i}\right)\right)\right|^{2}\leq C_{\psi}\sum_{i}\sum_{j}c_{i}c_{j}K\left(x_{i},x_{j}\right),\label{eq:st7}
\end{equation}
and, so by \lemref{B1} applied to $F=\psi\circ f$, conclusion in
(\ref{eq:st3b}) follows.

(\ref{eq:st3b})$\Rightarrow$(\ref{eq:st3a}). Assume (\ref{eq:st3b}),
we then reverse the above reasoning to get 
\begin{equation}
\Big|\Big\langle T_{f}\underset{\in\mathscr{D}_{K}}{\underbrace{\left(\sum_{i}c_{i}K_{x_{i}}\right)}},\psi\Big\rangle_{\mathscr{H}_{L}}\Big|\leq\sqrt{C_{\psi}}\left\Vert \sum_{i}c_{i}K_{x_{i}}\right\Vert _{\mathscr{H}_{K}}\label{eq:st8}
\end{equation}
which states that $\psi\in dom(T_{f}^{*})$, which is condition (\ref{eq:st3a}).
Now combine this with (\ref{eq:st6}), and we conclude that (\ref{eq:st4})
is satisfied for $\psi$, i.e., that $T_{f}^{*}\psi=\psi\circ f$
holds. 
\end{proof}

\subsection{Basis approach}

Let $X,Y,K,L,f$ be as usual, and define $T_{f}:\mathscr{H}_{K}\rightarrow\mathscr{H}_{L}$.
Since $K$ is p.d. on $X\times X$, the RKHS $\mathscr{H}_{K}$ allows
an ONB $\left\{ h_{i}\right\} _{i\in\mathbb{N}}$, $h_{i}\in\mathscr{H}_{K}$;
by general theory, we get the pointwise formula:
\begin{equation}
K\left(x_{1},x_{2}\right)=\sum_{i\in\mathbb{N}}h_{i}\left(x_{1}\right)h_{i}\left(x_{2}\right).\label{eq:bd1}
\end{equation}
Then our condition in \thmref{mc} is equivalent to the following:
\begin{gather*}
\left(X\ni x\longmapsto L\left(f\left(x\right),y\right)\right)\in\mathscr{H}_{K}\tag{{a}}\\
\Updownarrow\\
\sum_{i\in\mathbb{N}}\left|\left(T_{f}\left(h_{i}\right)\right)\left(y\right)\right|^{2}<\infty.\tag{{b}}
\end{gather*}

\begin{proof}
$\left(a\right)\Rightarrow\left(b\right)$ is detailed below; but
the converse implication will follow by the same argument. So by $\left(a\right)$,
$L_{y}\in dom(T_{f}^{*})$ and therefore:
\begin{equation}
T_{f}^{*}\left(L_{y}\right)\left(\cdot\right)\in\mathscr{H}_{K}.\label{eq:bd2}
\end{equation}
Since $\left\{ h_{i}\right\} $ is an ONB in $\mathscr{H}_{K}$, 
\begin{equation}
\sum_{i}\Big|\big\langle h_{i},\underset{\in\mathscr{H}_{K}}{\underbrace{T_{f}^{*}\left(L_{y}\right)}}\big\rangle\Big|^{2}=\left\Vert T_{f}^{*}\left(L_{y}\right)\right\Vert _{\mathscr{H}_{K}}^{2}<\infty.\label{eq:bd3}
\end{equation}
But from the $\text{LHS}_{\left(\ref{eq:bd3}\right)}$: 
\[
\left\langle h_{i},T_{f}^{*}\left(L_{y}\right)\right\rangle _{\mathscr{H}_{K}}=\left\langle T_{f}\left(h_{i}\right),L_{y}\right\rangle _{\mathscr{H}_{L}}=\left(T_{f}\left(h_{i}\right)\right)\left(y\right),
\]
and $\left(b\right)$ follows. 
\end{proof}
\textbf{Key Question: When is $F_{f,y}\left(\cdot\right)\in\mathscr{H}_{K}$?}
The cleanest answer to the question of what functions $X\xrightarrow{\;f\;}Y$
have the property that 

\begin{equation}
F_{f,y}\left(x\right)=L\left(f\left(x\right),y\right)\text{ is in \ensuremath{\mathscr{H}_{K}} }\label{eq:cd1}
\end{equation}
is this: 
\begin{thm}
\label{thm:mc2}Let $K,L$ and $f$ be given, then 
\begin{equation}
F_{f,y}\:\text{in }\left(\ref{eq:cd1}\right)\;\text{is in \ensuremath{\mathscr{H}_{K}\Longleftrightarrow L_{y}\in}dom(\ensuremath{T_{f}^{*}}),}\label{eq:cd2}
\end{equation}
where the operator $T_{f}:\mathscr{H}_{K}\rightarrow\mathscr{H}_{L}$
is given by 
\[
T_{f}\left(K_{x}\right):=L\left(f\left(x\right),\cdot\right).
\]
Moreover, (\ref{eq:cd2}) holds for all $y\in Y\Longleftrightarrow T_{f}$
is closable. 
\end{thm}

\subsection{Dual pairs of operators}

Consider a symmetric pair of operators with dense domains:
\[
\xymatrix{\mathscr{H}_{K}\ar@/{}^{1.3pc}/[rr]^{T_{f}} &  & \mathscr{H}_{L}\ar@/{}^{1.3pc}/[ll]^{T_{f}^{*}}}
\]
($T=T_{f}$, since it will depend on $f$) where 
\begin{equation}
span\left\{ K_{x}\right\} _{x\in X}\:\text{is dense in \ensuremath{\mathscr{H}_{K}}}\label{eq:dd2}
\end{equation}
and 
\begin{equation}
span\left\{ L_{y}\right\} _{y\in Y}\:\text{is dense in \ensuremath{\mathscr{H}_{L}}}\label{eq:dd3}
\end{equation}
such that 
\begin{align}
K_{x} & \in dom(T_{f}),\;\text{and}\label{eq:dd4}\\
L_{y} & \in dom(T_{f}^{*})\label{eq:dd5}
\end{align}
where ``$dom$'' denotes the respective operator domains. 
\begin{note*}
We note that 
\[
T_{f}\left(K_{x}\right)\left(\cdot\right)=L\left(f\left(x\right),\cdot\right)\in\mathscr{H}_{L}
\]
is always well defined, with dense domain, but the secret is $T_{f}^{*}.$ 

Also note that (\ref{eq:dd5}) is the condition in \thmref{mc}.
\end{note*}
Let $f:X\rightarrow Y$ be as before, and the two p.d. kernels $K$
and $L$ are fixed. We then introduce the corresponding (densely defined)
operator $T_{f}:\mathscr{H}_{K}\rightarrow\mathscr{H}_{L}$ by setting
\begin{equation}
T_{f}\left(K_{x}\right)=L\left(f\left(x\right),\cdot\right)\in\mathscr{H}_{L}.\label{eq:cc1}
\end{equation}

\textbf{Notation and convention. }$K_{x}\left(\cdot\right)$ is the
kernel function in $\mathscr{H}_{K}$ as usual: 
\begin{align}
K_{x}\left(t\right) & =K\left(x,t\right),\;\forall t\in X\:\text{and similarly,}\label{eq:cc2}\\
L_{y}\left(u\right) & =L\left(y,u\right),\;\forall u\in Y.\label{eq:cc3}
\end{align}

\begin{lem}
\label{lem:b9}When $L_{y}\in dom(T_{f}^{*})$ then 
\begin{equation}
\left(T_{f}^{*}\left(L_{y}\right)\right)\left(x\right)=L\left(f\left(x\right),y\right)\;\text{on \ensuremath{X},}\label{eq:cc4}
\end{equation}
equivalently, 
\begin{equation}
T_{f}^{*}\left(L_{y}\right)\left(\cdot\right)=L\left(f\left(\cdot\right),y\right)\;\text{on \ensuremath{X}.}\label{eq:cc4'}
\end{equation}
\end{lem}

\begin{proof}[Proof of (\ref{eq:cc4})]
 The conclusion (\ref{eq:cc4}) is equivalent to the following assertion:
\[
\xymatrix{\xyC{0pc}\Big\langle\underset{L\left(f\left(x\right),\cdot\right)}{\underbrace{T_{f}\left(K_{x}\right)\left(\cdot\right)}},L_{y}\Big\rangle_{\mathscr{H}_{L}}\ar[dr] & = & \Big\langle K_{x},\stackrel[T_{f}^{*}\left(L_{y}\right)]{\in\mathscr{H}_{K}}{\underbrace{\overbrace{L\left(f\left(\cdot\right),y\right)}}}\Big\rangle_{\mathscr{H}_{K}}\ar[dl]\\
 & L\left(f\left(x\right),y\right)
}
\]
The conclusion (\ref{eq:cc4}) follows since the respective kernel
functions span dense subspaces.
\end{proof}
Recall,

\begin{align*}
\text{the function }x & \longmapsto F_{f,y}\left(x\right)=L\left(f\left(x\right),y\right)\in\mathscr{H}_{K}\\
 & \Updownarrow\\
L_{y} & \in dom(T_{f}^{*}).
\end{align*}
Assume that $L_{y}\in dom(T_{f}^{*})$ , then apply the condition
for functions in $\mathscr{H}_{K}$ to $F_{f,y}\left(\cdot\right)$,
so $\forall n$, $\forall\left(x_{i}\right)_{1}^{n}$, $\forall\left(c_{i}\right)_{1}^{n}$,
$c_{i}\in\mathbb{R}$: 
\begin{eqnarray*}
\left|\sum_{i}c_{i}F_{f,y}\left(x_{i}\right)\right|^{2} & = & \left|\sum_{i}c_{i}L\left(f\left(x_{i}\right),y\right)\right|^{2}\\
 & \leq & \left|\left\langle \sum_{i}c_{i}K_{x_{i}},T_{f}^{*}\left(L_{y}\right)\right\rangle _{\mathscr{H}_{K}}\right|^{2}\\
 & \underset{\text{Schwarz}}{\leq} & \left\Vert \sum_{i}c_{i}K_{x_{i}}\right\Vert _{\mathscr{H}_{K}}^{2}\left\Vert T_{f}^{*}\left(L_{y}\right)\right\Vert _{\mathscr{H}_{K}}^{2}\\
 & = & \sum_{i}\sum_{j}c_{i}c_{j}K\left(x_{i},x_{j}\right)\left\Vert T_{f}^{*}\left(L_{y}\right)\right\Vert _{\mathscr{H}_{K}}^{2}.
\end{eqnarray*}

\begin{lem}
The implication below is both directions:
\begin{gather*}
X\ni x\longmapsto L\left(f\left(x\right),y\right)\in\mathscr{H}_{K}\;\text{for \ensuremath{\forall y}}\\
\Updownarrow\\
\text{the condition in \ensuremath{\left(\ref{thm:mc}\right)} is satisfied}.
\end{gather*}
Even if we fix $y\in Y$, then 
\begin{equation}
L_{y}\in dom(T_{f}^{*})\Longleftrightarrow\left(x\longmapsto L\left(f\left(x\right),y\right)\right)\in\mathscr{H}_{K}.\label{eq:f8}
\end{equation}
\end{lem}

\begin{proof}[Proof sketch]
 By definition, $L_{y}\in dom(T_{f}^{*})$, $\exists C_{y}<\infty$
$\Longleftrightarrow$
\[
\left|\left\langle T_{f}\varphi,L_{y}\right\rangle _{\mathscr{H}_{L}}\right|=\left|\left(T_{f}\left(\varphi\right)\right)\left(y\right)\right|\leq C_{y}\left\Vert \varphi\right\Vert _{\mathscr{H}_{K}}
\]
holds for $\forall\varphi\in span\left\{ K_{x}\right\} _{x\in X}.$
But 
\begin{align}
T_{f}\left(K_{x}\right)\left(y\right) & =L\left(f\left(x\right),y\right),\;\text{and}\label{eq:f9}\\
\Big|\sum_{i}c_{i}\underset{F_{f,y}\left(x_{i}\right)}{\underbrace{L\left(f\left(x_{i}\right),y\right)}}\Big| & ^{2}=\left|\Big\langle\sum_{i}c_{i}K_{x_{i}},T_{f}^{*}\left(L_{y}\right)\Big\rangle_{\mathscr{H}_{K}}\right|^{2}\\
 & \leq\left\Vert T_{f}^{*}\left(L_{y}\right)\right\Vert _{\mathscr{H}_{K}}^{2}\underset{<\infty}{\underbrace{\sum_{i}\sum_{j}c_{i}c_{j}K\left(x_{i},x_{j}\right)}}
\end{align}
and so by the basic lemma for $\mathscr{H}_{K}$ (see the proof of
\thmref{mc}), we conclude that functions $F_{f,y}\in\mathscr{H}_{K}$,
i.e., $\left(x\longmapsto L\left(f\left(x\right),y\right)\right)\in\mathscr{H}_{K}$.\textbf{ }

\textbf{Conclusion: }the bi-implication $\Longleftrightarrow$ in
(\ref{eq:f8}) is valid.
\end{proof}

\subsection{Functions and Operators}

In general if $T:\mathscr{H}_{1}\rightarrow\mathscr{H}_{2}$ is an
operator with dense domain $\mathscr{D}\subset\mathscr{H}_{1}$, where
$\mathscr{H}_{i}$, $i=1,2$, are two Hilbert spaces, we know that
$T$ is closable $\Longleftrightarrow$ $T^{*}$ is densely defined,
i.e., iff $dom(T^{*})$ is dense in $\mathscr{H}_{2}$ (see e.g.,
\cite{MR4274591}). So we apply this to $T=T_{f}$, $\mathscr{H}_{1}=\mathscr{H}_{K}$,
$\mathscr{H}_{2}=\mathscr{H}_{L}$, and the condition in \thmref{mc}
holds $\Longleftrightarrow$ $L_{y}\in dom(T_{f}^{*})$ $\forall y\in Y$.
Since $span\left\{ L_{y}\right\} _{y\in Y}$ is dense in $\mathscr{H}_{L}$,
the condition in \thmref{mc} $\Longrightarrow$ $T_{f}$ is closable. 

Given $K$ and $L$ as above, introduce 
\begin{align}
\mathscr{F}_{ub}\left(K,L\right) & =\left\{ f:T_{f}\text{ is closable}\right\} ,\;\text{and}\label{eq:ff6}\\
\mathscr{F}_{b}\left(K,L\right) & =\left\{ f:T_{f}\text{ is bounded from \ensuremath{\mathscr{H}_{K}} into \ensuremath{\mathscr{H}_{L}}}\right\} .\label{eq:ff7}
\end{align}
In both cases, the operators $T=T_{f}$ depends on the choice of function
$X\xrightarrow{\;f\;}Y$, but the two conditions (\ref{eq:ff6}) and
(\ref{eq:ff7}) are different:
\begin{equation}
\left(T_{f}\left(K_{x}\right)\right)\left(y\right)=L\left(f\left(x\right),y\right)=\left(\left(T_{f}\right)^{*}\left(L_{y}\right)\right)\left(x\right),\label{eq:ff8}
\end{equation}
for all $x\in X$, and $y\in Y$. See details below:

Some general comments about the operator $T_{f}:\mathscr{H}_{K}\rightarrow\mathscr{H}_{L}$.
As before, $K$ and $L$ are fixed p.d. kernels, and $f:X\rightarrow Y$
is a function. We need to understand the conclusion from \thmref{mc},
i..e, when is 
\begin{equation}
\left(X\ni x\longmapsto L\left(f\left(x\right),y\right)\right)\in\mathscr{H}_{K}\;\text{for all \ensuremath{y\in Y}}?\label{eq:cao1}
\end{equation}

Answer: (\ref{eq:cao1}) $\Longleftrightarrow$ $L_{y}\in dom(T_{f}^{*})$. 

Note that then the function in (\ref{eq:cao1}) is $T_{f}^{*}\left(L_{y}\right)$;
see (\ref{eq:ff8}). But note that, starting with a function $X\xrightarrow{\;f\;}Y$,
there are requirements for having (\ref{eq:ff8}) yield a well defined
linear operator $T_{f}$ with dense domain in $\mathscr{H}_{K}$,
s.t. 
\begin{equation}
T_{f}\left(K_{x}\right)\left(\cdot\right)=L\left(f\left(x\right),\cdot\right).\label{eq:cao9}
\end{equation}
The case when $T_{f}$ is bounded is easy since then $dom(T_{f}^{*})=\mathscr{H}_{L}$.
Notationally, $L\left(f\left(x\right),\cdot\right)\in L_{f\left(x\right)}\in\mathscr{H}_{L}$,
but we must also verify the implicit kernel function for all finite
sums:
\[
\sum_{i}\sum_{j}c_{i}c_{j}K\left(x_{i},x_{j}\right)=0\Longrightarrow\sum_{i}\sum_{j}c_{i}c_{j}L\left(f\left(x_{i}\right),f\left(x_{j}\right)\right)=0.
\]

\subsection{The case when $K=L$}

As demonstrated in \secref{kact} below, for applications to multi-level
NNs, the recursive constructions simplify when the same p.d. kernel
$K$ is used at each level. Hence below, we specialize to the case
when $X=Y$, and $K=L$; see the setting in Theorems \ref{thm:mc}
and \ref{thm:mc2}.
\begin{thm}
\label{thm:b9}Consider a positive definite kernel $K$ on $X\times X$,
and the corresponding RKHS $\mathscr{H}_{K}$, i.e., the Hilbert completion
of $\left\{ K_{x}\right\} _{x\in X}$ where $K_{x}:=K\left(\cdot,x\right)$.
Fix a function $X\xrightarrow{\;f\;}X$ with the property (see \thmref{mc})
that 
\begin{equation}
\left(X\ni x\longmapsto K\left(f\left(x\right),y\right)\right)\in\mathscr{H}_{K}\;\text{for all \ensuremath{y\in X}.}\label{eq:KL1}
\end{equation}
Hence, the operator $T_{f}:\mathscr{H}_{K}\rightarrow\mathscr{H}_{K}$
defined by 
\begin{equation}
T_{f}\left(K\left(\cdot,y\right)\right):=K\left(f\left(\cdot\right),y\right)\label{eq:KL2}
\end{equation}
is a densely defined operator from $\mathscr{H}_{K}$ into $\mathscr{H}_{K}$,
with domain 
\begin{equation}
\mathscr{D}_{K}:=span\left\{ K_{x}\right\} _{x\in X}.\label{eq:KL3}
\end{equation}
 
\begin{enumerate}
\item \label{enu:KL1}Then the closure of $T_{f}$ (also denoted $T_{f}$)
is well defined and normal, i.e., the two operators $T_{f}$ and $T_{f}^{*}$
commute. 
\item \label{enu:KL2}In particular, $T_{f}$ has a projection-valued spectral
resolution, i.e., there is a projection-valued measure $Q\left(\cdot\right)$
on $\mathscr{B}_{\mathbb{C}}\left(=\text{the Borel subsets in \ensuremath{\mathbb{C}}}\right)$
such that 
\begin{equation}
T_{f}=\int_{spect\left(T_{f}\right)}\lambda\,Q\left(d\lambda\right):\mathscr{D}_{K}\rightarrow\mathscr{H}_{K}.\label{eq:KL4}
\end{equation}
\end{enumerate}
\end{thm}

\begin{proof}
Note that part (\ref{enu:KL2}) follows from (\ref{enu:KL1}) and
the Spectral Theorem for normal operators (in the Hilbert space $\mathscr{H}_{K}$.)

Part (\ref{enu:KL1}). When the operator $T_{f}^{*}$ is introduced,
we get the following commutativity: 
\begin{figure}[H]
\[
\xymatrix{\xyC{4pc}K\left(\cdot,x\right)\ar[r]\sb(0.4){T_{f}}\ar[d]_{T_{f}^{*}} & K\left(\cdot,f\left(x\right)\right)\ar[d]^{T_{f}^{*}}\in\mathscr{D}_{K}\\
K\left(f\left(\cdot\right),x\right)\ar[r]\sb(0.4){T_{f}} & K\left(f\left(\cdot\right),f\left(x\right)\right)\in\mathscr{H}_{K}
}
\]

\caption{Commutativity of $T_{f}$ and $T_{f}^{*}$.}
\end{figure}
which is the desired conclusion (\ref{enu:KL1}).
\end{proof}
Given a function $f:X\rightarrow X$ as in \thmref{b9}. Below we
make use of the corresponding projection valued measure $Q^{\left(f\right)}$
from \thmref{b9} in order to establish an assignment from pairs $(x,y)$
of points in $X$, into systems of complex measures $\mu_{\left(x,y\right)}$
on the spectrum of $T_{f}$. In this assignment, $n$-fold composition-iteration
of the function $f$ yields the $n$th moment of each of the measures
$\mu_{\left(x,y\right)}$.
\begin{cor}
\label{cor:b10}Let $K,X$ and $f$ be as specified in \thmref{b9},
and let $Q=Q^{\left(K,f\right)}\left(\cdot\right)$ be the corresponding
projection valued measure in (\ref{eq:KL4}). Then for every pair
$x,y\in X$, we get a corresponding Borel measure 
\begin{equation}
\mu_{x,y}^{\left(f\right)}\left(B\right)=\left\langle K_{x},Q\left(B\right)K_{y}\right\rangle _{\mathscr{H}_{K}}=\left(Q\left(B\right)\left(K_{y}\right)\right)\left(x\right),
\end{equation}
for all $B\in\mathscr{B}_{\mathbb{C}}$. Inductively, setting 
\[
f^{\circ n}=\underset{\text{\ensuremath{n} fold}}{\underbrace{f\circ\cdots\circ f}}
\]
we arrive at the following moment formula for the respective complex
measures: 
\begin{equation}
\mu_{f^{\circ n}\left(x\right),y}^{\left(f\right)}\left(B\right)=\int_{B}\lambda^{n}\,\mu_{x,y}^{\left(f\right)}\left(d\lambda\right).
\end{equation}
\end{cor}

We now turn to the role of \emph{multipliers} in the RKHS $\mathscr{H}_{K}$. 
\begin{defn}
A scalar valued function $\varphi$ on $X$ is said to be a \emph{multiplier}
for $\mathscr{H}_{K}$ iff one of the two equivalent conditions hold: 
\begin{enumerate}
\item \label{enu:m1}The multiplication operator $M_{\varphi}$ acting on
$\mathscr{H}_{K}$ via $M_{\varphi}F=\varphi F$ (via pointwise product)
leaves $\mathscr{H}_{K}$ invariant.
\item \label{enu:m2}We have the following identity for the adjoint operator:
\begin{equation}
M_{\varphi}^{*}\left(K_{x}\right)=\varphi\left(x\right)K_{x}\;\text{for all \ensuremath{x\in X}}\label{eq:KL6}
\end{equation}
where $K_{x}$ denotes the kernel function $K_{x}=K\left(\cdot,x\right)$. 
\end{enumerate}
\end{defn}

\begin{rem}
The equivalence of (\ref{enu:m1}) and (\ref{enu:m2}) follows from
the standard reference on RKHSs; see e.g., \cite{MR4274591}.
\end{rem}

\begin{thm}
\label{thm:b15}Let $K$ be a fixed p.d. kernel on $X\times X$, and
let $\mathscr{H}_{K}$ be the corresponding RKHS. Let $X\xrightarrow{\;f\;}X$
be a function such that (\ref{eq:KL1}) holds, i.e., $\left(X\ni x\longmapsto K\left(f\left(x\right),y\right)\right)\in\mathscr{H}_{K}$
for all $y\in X$. 

Then for every multiplier $\varphi$ for $\mathscr{H}_{K}$, we have:
\begin{equation}
M_{\left(\varphi\circ f\right)}T_{f}^{*}=T_{f}^{*}M_{\varphi}.\label{eq:KL7}
\end{equation}
 
\end{thm}

\begin{proof}
It is clear that the conclusion (\ref{eq:KL7}) has the following
equivalent form: 
\begin{equation}
T_{f}M_{\left(\varphi\circ f\right)}^{*}=M_{\varphi}^{*}T_{f};\label{eq:KL8}
\end{equation}
and below we shall prove (\ref{eq:KL8}). 

Let $f$ and $\varphi$ be as specified in the theorem. We then have
the following commutative diagram: 
\begin{figure}[H]
\[
\xymatrix{\xyC{4pc}K\left(\cdot,x\right)\ar[r]\sb(0.4){M_{\left(\varphi\circ f\right)}^{*}}\ar[d]_{T_{f}} & \varphi\left(f\left(x\right)\right)K\left(\cdot,x\right)\ar[d]^{T_{f}}\\
K\left(\cdot,f\left(x\right)\right)\ar[r]\sb(0.4){M_{\varphi}^{*}} & \varphi\left(f\left(x\right)\right)K\left(\cdot,f\left(x\right)\right)
}
\]

\caption{\label{fig:KL2} Commutative diagram corresponding to (\ref{eq:KL8}).}
\end{figure}

In the verification of the assertions in \figref{KL2}, we used the
conclusions in Theorems \ref{thm:mc} and \ref{thm:mc2} above. 
\end{proof}

\section{\label{sec:kact}Neural Network-activation functions from p.d. kernels}

In the previous section we introduced the use of positive definite
kernels, and associated generating function for the NN algorithms.
Below we make use of the kernel analysis in design of the generating
NN functions.

The next definition makes use of the iterative generation of feedforward
functions as in the literature, e.g., \cite{MR4185345,MR4399726,MR3457582}.
The recursive steps used here in the definition and \lemref{c2} below
serve as applications of our general framework from Theorems \ref{thm:mc}
and \ref{thm:mc2} above.
\begin{defn}
Let $K$ be a positive definite kernel on $\mathbb{R}$. An $l$-layer
feedforward network with kernel $K$ is a function of the form 
\begin{multline*}
x\mapsto y_{1}=K\left(A_{1}x+b_{1},c_{1}\right)\mapsto y_{2}=K\left(A_{2}y_{1}+b_{2},c_{2}\right)\mapsto\cdots\\
\cdots\mapsto y_{l}=K\left(A_{l}y_{l-1}+b_{l},c_{l}\right)\mapsto y_{out}=K\left(\left\langle a_{l+1},y_{l}\right\rangle +b_{l+1},c_{l+1}\right)
\end{multline*}
where 
\begin{itemize}
\item $x\in\mathbb{R}^{n_{0}}$;
\item $A_{j}\in\mathbb{R}^{n_{j}\times n_{j-1}}$, $b_{j},c_{j}\in\mathbb{R}^{n_{j}}$
for $j=1,\cdots,l$;
\item $a_{l+1}\in\mathbb{R}^{n_{l}}$, $b_{l+1},c_{l+1}\in\mathbb{R}$;
\end{itemize}
and for vectors $x,y\in\mathbb{R}^{m}$, 
\[
K\left(x,y\right):=\left[K\left(x_{1},y_{1}\right),\cdots,K\left(x_{m},y_{m}\right)\right].
\]

\end{defn}

\begin{lem}
\label{lem:c2}Let $K\left(x,y\right)=\min\left(x,y\right)$, and
$a,b,c,d$ be nonzero constants. Then
\begin{enumerate}
\item $K\left(ax+b,c\right)=aK\left(x,a^{-1}\left(c-b\right)\right)+b$; 
\item $K\left(K\left(x,a\right),b\right)=K\left(x,K\left(a,b\right)\right)$; 
\item $K\left(dK\left(ax+b,c\right)+e,f\right)=daK\left(x,K\left(a^{-1}\left(c-b\right),a^{-1}\left(d^{-1}\left(f-e\right)-b\right)\right)\right)+db+e$.
\end{enumerate}
\end{lem}

\begin{proof}
~
\begin{enumerate}
\item $K\left(ax+b,c\right)=\begin{cases}
ax+b & x<a^{-1}\left(c-b\right)\\
c & x>a^{-1}\left(c-b\right)
\end{cases}$
\item Assume $a<b$, then 
\[
K\left(K\left(x,a\right),b\right)=\begin{cases}
x & x<a\\
a & x\geq a
\end{cases}=K\left(x,a\right).
\]
The case $a>b$ is similar. 
\item This follows from (1)--(2):
\begin{eqnarray*}
 &  & K\left(dK\left(ax+b,c\right)+e,f\right)\\
 & = & dK\left(K\left(ax+b,c\right),d^{-1}\left(f-e\right)\right)+e\\
 & = & dK\left(aK\left(x,a^{-1}\left(c-b\right)\right)+b,d^{-1}\left(f-e\right)\right)+e\\
 & = & d\left\{ aK\left(K\left(x,a^{-1}\left(c-b\right)\right),a^{-1}\left(d^{-1}\left(f-e\right)-b\right)\right)+b\right\} +e\\
 & = & daK\left(K\left(x,a^{-1}\left(c-b\right)\right),a^{-1}\left(d^{-1}\left(f-e\right)-b\right)\right)+db+e\\
 & = & daK\left(x,K\left(a^{-1}\left(c-b\right),a^{-1}\left(d^{-1}\left(f-e\right)-b\right)\right)\right)+db+e
\end{eqnarray*}
\end{enumerate}
\end{proof}
In what follows, all the networks are restricted to be defined on
compact subsets $\Omega$ in $\mathbb{R}^{d}$, e.g., $\Omega=\left[0,1\right]^{d}$
(hypercubes). This is in consideration of standard normalizations
in training neural networks. 

In Theorems \ref{thm:hk} and \ref{thm:hmu}, we present in detail
the particular \emph{relative Reproducing Kernel Hilbert Spaces} which
have as their respective dipole system (see (\ref{eq:c3})) the generalized
ReLu functions illustrated here in Figures \ref{fig:dp} and \ref{fig:rmu}.

Here we specify the kernel $K_{1}$ for Brownian motion $W$ indexed
by $\mathbb{R}$. As a result, the corresponding p.d. kernel on $\mathbb{R}\times\mathbb{R}$
is as follows: 
\begin{equation}
K_{1}\left(x,y\right)=\begin{cases}
\left|x\right|\wedge\left|y\right|=\min\left(\left|x\right|,\left|y\right|\right) & \text{if}\;xy\geq0\;\text{(so same sign)}\\
0 & \text{if \ensuremath{xy<0,} so opposite sign.}
\end{cases}\label{eq:c1}
\end{equation}

\begin{proof}
The connection between the kernel $K_{1}$ and the Brownian motion
$\left\{ W_{x}\right\} _{x\in\mathbb{R}}$ is as follows: 
\begin{equation}
K_{1}\left(x,y\right)=\mathbb{E}\left(\left(W_{x}-W_{0}\right)\left(W_{y}-W_{0}\right)\right)\label{eq:c2}
\end{equation}
for all $x,y\in\mathbb{R}$. The asserted formula (\ref{eq:c1}) follows
from this, combined with the independence of increments for Brownian
motion. 
\end{proof}
\begin{thm}
\label{thm:hk}Let $K_{1}$ be the p.d. kernel (\ref{eq:c1}) on $\mathbb{R}\times\mathbb{R}$,
with the corresponding RKHS 
\[
\mathscr{H}_{K_{1}}=\left\{ f:f'\in L^{2}\right\} ,\quad\left\Vert f\right\Vert _{\mathscr{H}_{K_{1}}}^{2}=\int\left|f'\right|^{2}d\lambda_{1}.
\]
On $\Omega=\left[0,1\right]^{d}$, consider the p.d. kernel 
\[
K_{d}\left(x,y\right)=K_{1}\left(x_{1},y_{1}\right)\cdots K_{1}\left(x_{d},y_{d}\right),
\]
so that 
\[
\mathscr{H}_{K_{d}}=\left\{ f:\nabla f\in L^{2}\right\} ,\quad\left\Vert f\right\Vert _{\mathscr{H}_{K_{d}}}^{2}=\int\left|\nabla f\right|^{2}d\lambda_{d},
\]
where $\lambda_{d}$ denotes the $d$-dimensional Lebesgue measure. 

Given $f:\Omega\rightarrow\mathbb{R}$, and a fixed $c\in\mathbb{R}$,
set 
\[
F:\mathbb{R}^{d}\rightarrow\mathbb{R}^{1},\quad F\left(x\right)=K_{1}\left(f\left(x\right),c\right).
\]
Then, 
\[
F\in\mathscr{H}_{K_{d}}\Longleftrightarrow\iint_{f^{-1}\left(\left[0,c\right]\right)}\left|\nabla f\right|^{2}d\lambda_{d}<\infty.
\]
\end{thm}

\begin{thm}
\label{thm:hmu}Let $\mu$ be a non-atomic $\sigma$-finite measure
on $\left(\mathbb{R},\mathscr{B}_{\mathbb{R}}\right)$, and consider
Stieltjes measures $dF$ on $\left(\mathbb{R},\mathscr{B}_{\mathbb{R}}\right)$
such that 
\begin{equation}
dF\ll\mu
\end{equation}
\textup{(absolutely continuous).} Then the relative RKHS $\mathscr{H}_{\mu}$
for the p.d. kernel 
\begin{equation}
K_{\mu}\left(A,B\right)=\mu\left(A\cap B\right)
\end{equation}
consists functions $F$ such that 
\begin{equation}
F\left(b\right)-F\left(a\right)=\left\langle v_{a,b}^{\left(\mu\right)},F\right\rangle _{\mathscr{H}_{\mu}}\label{eq:c3}
\end{equation}
\begin{equation}
\int_{\mathbb{R}}\left|\frac{dF}{d\mu}\right|^{2}d\mu<\infty
\end{equation}
where the relative kernels $v_{a,b}^{\left(\mu\right)}$ are as follows:
\[
\frac{dv_{a,b}^{\left(\mu\right)}}{d\mu}\left(x\right)=\chi_{\left[a,b\right]\left(x\right)},
\]
see \figref{rmu}. 
\end{thm}

\begin{figure}[H]
\includegraphics[width=0.5\textwidth]{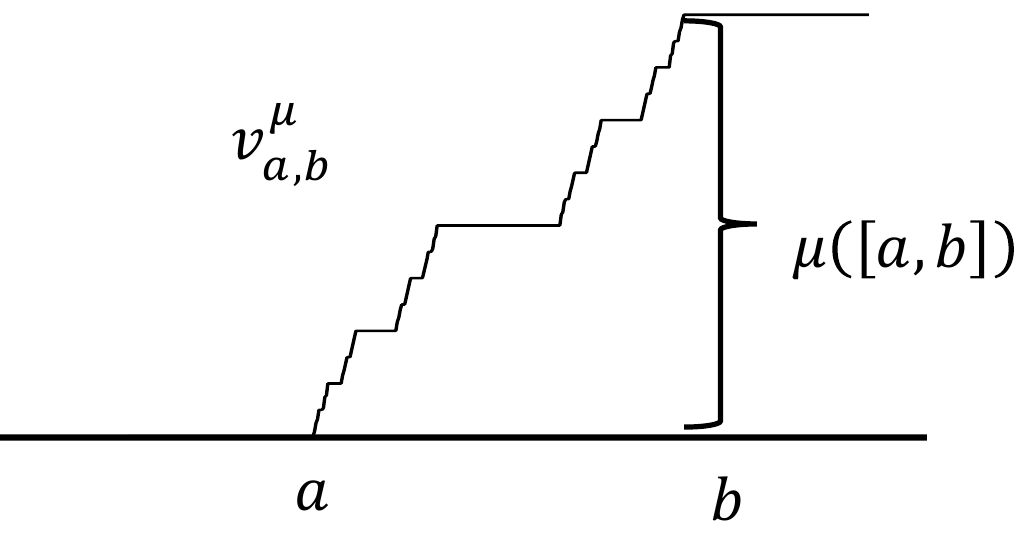}\caption{\label{fig:rmu}Illustration dipole functions that reproduce differences
of values of functions in the space $\mathscr{H}_{\mu}$. Compare
with \figref{dp} above.}

\end{figure}

\begin{proof}
See \cite{MR3251728,doi:https://doi.org/10.1002/9781119414421.ch2}
and the details in the proof of \thmref{hk}.
\end{proof}
\begin{rem}
The positive definite kernel $K_{\mu}$ which is ``responsible''
for the relative RKHS $\mathscr{H}_{\mu}$ is defined on $\mathscr{B}\times\mathscr{B}$,
where $\mathscr{B}$ denotes the Borel $\sigma$-algebra of subsets
of $\mathbb{R}$. Using \cite{doi:https://doi.org/10.1002/9781119414421.ch2},
one checks that 
\begin{equation}
K_{\mu}\left(A,B\right):=\mu\left(A\cap B\right)\;\text{for all \ensuremath{A,B\in\mathscr{B}.}}
\end{equation}
We further note that $K_{\mu}$ is the covariance for the generalized
$\mu$-Brownian motion $\{W_{A}^{\left(\mu\right)}\}_{A\in\mathscr{B}}$,
i.e., subject to 
\begin{equation}
\mathbb{E}\left(W_{A}^{\left(\mu\right)}W_{B}^{\left(\mu\right)}\right)=\mu\left(A\cap B\right)\;\text{for all \ensuremath{A,B}\ensuremath{\in\mathscr{B}}.}
\end{equation}
The corresponding Ito-lemma for $W^{\left(\mu\right)}$ is defined
for differentiable functions $f$ on $\mathbb{R}$ via 
\begin{equation}
f\left(W_{A}^{\left(\mu\right)}\right)-f\left(0\right)=\int_{A}f'\left(W_{t}^{\left(\mu\right)}\right)dW_{t}^{\left(\mu\right)}+\frac{1}{2}\int_{A}f''\left(W_{t}^{\left(\mu\right)}\right)\mu\left(dt\right).
\end{equation}
In particular, the measure $\mu$ is the quadratic variation of $W_{t}^{\left(\mu\right)}$. 
\end{rem}

\section{\label{sec:fwidth}Applications to fractal images}

In recent decades, it has become evident that fractal features arise
in diverse datasets, in time series and in image analysis, to mention
two. (See e.g., \cite{MR4472252,MR4320089}.) Perhaps the best known
examples of fractal features include precise symmetries of scales.
Via a prescribed system of affine maps, they take the form of self-similarity.
And a special case, includes iterated function systems (IFS), and
maximal-entropy measures, also called IFS measures. The more familiar
Cantor constructions, e.g., scaling by 3 or scaling by 4, are examples
of IFS measures. For each of these cases, the RKHS framework we present
in \secref{kact}, serve as ideal tools for such adapted NN algorithms.
In particular, this may be illustrated with large numbers of images,
say 5000 generated images, each one is a fractal, either 2D or 3D,
with random rotation, with zooming, and coloring; half of them have
scaling 3, the other half have scaling 4. This leads to training of
a network serving to classify the images by scaling factors.

In particular, the Cantor-type activation functions, or the cumulative
functions of Cantor-like measures (\figref{rmu}), have vanishing
derivatives over structured subintervals of $\left[0,1\right]$. This
feature may lead to several benefits in neural networks. For example,
such functions can introduce sparsity and regularization into the
network, which improves its generalization performance and reduces
the risk of overfitting. Additionally, these functions can make the
network more robust to noise and other perturbations in the input
data, which improves its performance on unseen data. Furthermore,
activation functions whose derivative is zero over subintervals allow
the network to learn more complex and non-linear patterns in the data.
This can improve the expressiveness and flexibility of the network,
making it more accurate and effective for a wider range of tasks.
Additionally, these functions can make the network easier to optimize
and train, since the gradient of the activation function is well-structured,
thus reduce the computational complexity and improve the convergence
rate of the training algorithm.

More generally, a neural network with a custom activation function
(see e.g. the dipoles in \figref{dp}) uses a non-standard activation
function with adjustable parameters that can be trained and optimized
during the learning process. This allows the network to learn more
complex and non-linear relationships between the input and output
data, which can improve the accuracy of the network's predictions.

The use of a custom activation function with trainable parameters
can be useful in a variety of applications, such as image recognition,
natural language processing, and time series forecasting. It can also
be used to improve the performance of other machine learning algorithms,
such as decision trees and support vector machines (see e.g., \cite{zbMATH01669138,MR4329806,MR2849119,MR3108145,MR2274418,MR2246374}).

Below we apply a custom activation function in a ConvNet to classify
fractal images. In this setting, the activation function should be
designed to capture the complex, self-similar patterns that are characteristic
of the fractal images. The network is trained on a dataset of fractal
images with corresponding labels. It is optimized using a gradient-based
algorithm, such as stochastic gradient descent. Once trained, the
network will be used to classify new fractal images and predict their
classes with high accuracy.

In the example below, a dataset\footnote{Available at https://www.kaggle.com/dsv/4791103.}
of 15,000 Cantor-like 3D images is generated in Mathematica. Parameters
of each image, such as zoom factor, viewing angle, and scaling factor,
are uniformly distributed. A sample of the images is shown in \figref{sample}. 

The images are split into three categories according to their scaling
factors, labeled as class ``1'', ``2'' and ``3'', respectively.
The entire dataset is divided into a training set (size = 10,000),
validation set (size = 2,500) and test set (size = 2,500). The task
is to train a ternary classifier using the training set, along with
the validation set (for model selection), whose performance is then
tested on the test images. 

In the experiment, a small ConvNet is implemented in Keras. Its architecture
is shown in \figref{cnn}. The loss and accuracy of the model are
recorded for 20 epochs (\figref{compare}).

In comparison with a standard Relu network (\figref{relu-act}) of
the same architecture, the use of Cantor-like activation is better
at reducing overfitting (\figref{cantor-act}); it is expected that
with a systematic hyperparameter tuning, such a network has the potential
to outperform Relu networks in certain applications. 

\begin{figure}[H]
\begin{tabular}[t]{>{\centering}p{0.1\columnwidth}>{\centering}m{0.25\columnwidth}>{\centering}m{0.25\columnwidth}>{\centering}m{0.25\columnwidth}}
Class 1 & \includegraphics[width=0.25\textwidth]{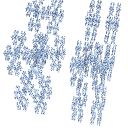} & \includegraphics[width=0.25\textwidth]{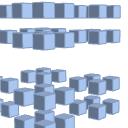} & \includegraphics[width=0.25\textwidth]{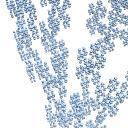}\tabularnewline
Class 2 & \includegraphics[width=0.25\textwidth]{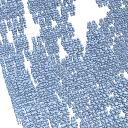} & \includegraphics[width=0.25\textwidth]{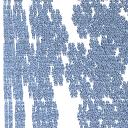} & \includegraphics[width=0.25\textwidth]{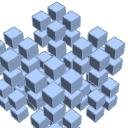}\tabularnewline
Class 3 & \includegraphics[width=0.25\textwidth]{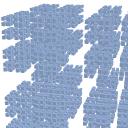} & \includegraphics[width=0.25\textwidth]{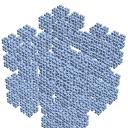} & \includegraphics[width=0.25\textwidth]{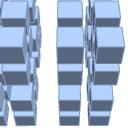}\tabularnewline
\end{tabular}\caption{\label{fig:sample} A random sample of the dataset of 3D Cantor images.}
\end{figure}

\begin{figure}[H]
\begin{tabular}{>{\raggedright}p{0.5\textwidth}>{\raggedright}p{0.25\textwidth}>{\raggedright}p{0.15\textwidth}}
Model: \textquotedbl Cantor-activation\textquotedbl{} &  & \tabularnewline
\hline 
Layer (type) & Output Shape & Param \#\tabularnewline
\hline 
input\_1 (InputLayer) & (None, 128, 128, 3) & 0\tabularnewline
rescaling (Rescaling) & (None, 128, 128, 3) & 0\tabularnewline
conv2d (Conv2D) & (None, 126, 126, 16) & 448\tabularnewline
max\_pooling2d (MaxPooling2D) & (None, 63, 63, 16) & 0\tabularnewline
conv2d\_1 (Conv2D) & (None, 61, 61, 32) & 4640\tabularnewline
max\_pooling2d\_1 (MaxPooling2D) & (None, 30, 30, 32) & 0\tabularnewline
conv2d\_2 (Conv2D) & (None, 28, 28, 64) & 18496\tabularnewline
max\_pooling2d\_2 (MaxPooling2D) & (None, 14, 14, 64) & 0\tabularnewline
conv2d\_3 (Conv2D) & (None, 12, 12, 128) & 73856\tabularnewline
flatten (Flatten) & (None, 18432) & 0 		 	\tabularnewline
dense (Dense) & (None, 3) & 55299\tabularnewline
\hline 
\end{tabular}

\caption{\label{fig:cnn} A ConvNet for fractal image classification.}
\end{figure}

\begin{figure}[H]
\subfloat[\label{fig:relu-act}Relu activation]{%
\begin{tabular}{c}
\includegraphics[width=0.45\textwidth]{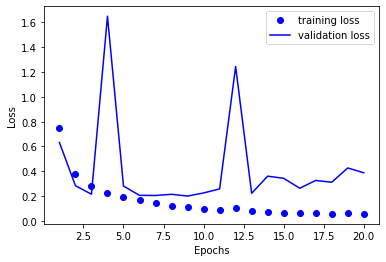}\tabularnewline
\includegraphics[width=0.45\textwidth]{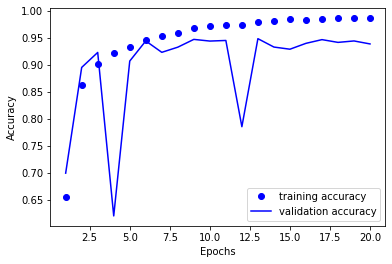}\tabularnewline
\end{tabular}

}\hfill{}\subfloat[\label{fig:cantor-act}Cantor-like activation]{%
\begin{tabular}{c}
\includegraphics[width=0.45\textwidth]{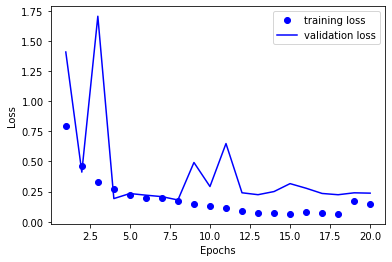}\tabularnewline
\includegraphics[width=0.45\textwidth]{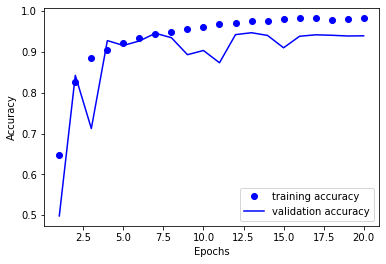}\tabularnewline
\end{tabular}

}

\caption{\label{fig:compare} Training loss and validation loss, illustrated
with use of a ConvNet.}
\end{figure}

\pagebreak{}

\bibliographystyle{amsalpha}
\bibliography{ref}

\end{document}